\documentclass[conference]{IEEEtran}
\IEEEoverridecommandlockouts
\usepackage{cite}
\usepackage{amsmath,amssymb,amsfonts}
\usepackage{algorithmic}
\usepackage{graphicx}
\usepackage{textcomp}
\usepackage{xcolor}
\def\BibTeX{{\rm B\kern-.05em{\sc i\kern-.025em b}\kern-.08em
    T\kern-.1667em\lower.7ex\hbox{E}\kern-.125emX}}

\usepackage{cite}
\usepackage{hyperref}
\usepackage{dsfont}
\usepackage{bm}
\usepackage{amsthm}
\usepackage{color}
\usepackage{comment}

\DeclareMathOperator{\bcw}{{\boldsymbol{\scriptstyle\mathcal{W}}}}

\newtheorem{theorem}{Theorem}
\newtheorem{assumption}{Assumption}
\newtheorem{lemma}{Lemma}

\usepackage{accents}
\newcommand{\dbtilde}[1]{\accentset{\approx}{#1}}

\begin{document}
\bstctlcite{IEEEexample:BSTcontrol}

\title{Deep-Relative-Trust-Based Diffusion for Decentralized Deep Learning\\
\thanks{This work was supported by EPSRC Grants EP/X04047X/1 and EP/Y037243/1. Emails: \{muyun.li23, aaron.fainman22, s.vlaski\}@imperial.ac.uk.}
}

\author{\IEEEauthorblockN{Muyun Li, Aaron Fainman and Stefan Vlaski}
\IEEEauthorblockA{\textit{Department of Electrical and Electronic Engineering, Imperial College London}}
}

\maketitle

\begin{abstract}
Decentralized learning strategies allow a collection of agents to learn efficiently from local data sets without the need for central aggregation or orchestration. Current decentralized learning paradigms typically rely on an averaging mechanism to encourage agreement in the parameter space. We argue that in the context of deep neural networks, which are often over-parameterized, encouraging consensus of the neural network outputs, as opposed to their parameters can be more appropriate. This motivates the development of a new decentralized learning algorithm, termed DRT diffusion, based on deep relative trust (DRT), a recently introduced similarity measure for neural networks. We provide convergence analysis for the proposed strategy, and numerically establish its benefit to generalization, especially with sparse topologies, in an image classification task.
\end{abstract}


\section{Introduction}
\label{sec:intro}
We consider a collection of \( K \) agents, each equipped with a neural network of identical architecture consisting of \( L \) layers. We denote the parameters of the \( p \)-th layer of agent \( k \) by \( w_{k}^{(p)} \), and collect for brevity \( w_{k} \triangleq \mathrm{col}\{ w_{k}^{(p)} \}_{p=1}^P \). We denote the output of the neural network with parameterization \( w_{k} \) to the input \( x \) by \( f(x;w_{k}) \). To train the parameters of the network to fit a random pair of inputs \( \bm{x} \) and outputs \( \bm{y} \), we define a loss function \( Q(w_k; \bm{x}, \bm{y}) \), giving rise to the risk:
\begin{align}
    J_k(w_k) = \mathds{E} Q(w_k; \bm{x}, \bm{y})
\end{align}
Most algorithms for decentralized learning pursue solutions to the consensus optimization problem~\cite{Nedic:2009,Sayed:2014,Lian:2017,Vlaski23}:
\begin{equation}
    \label{eqn:aggObjec}
    \min_w J(w) \quad \text{where } J(w) \triangleq \frac{1}{K} \sum_{k=1}^{K} J_k(w)
\end{equation}
For the purposes of motivating this work, it will be sufficient to review the diffusion algorithm \cite{Sayed:2014}:
\begin{subequations}
	\begin{align}
		\bm{\psi}_{k, i} & =\bm{w}_{k, i-1}-\mu \widehat{\nabla J}_k\left(\bm{w}_{k, i-1}\right) \label{eqn:diff1}\\
		\bm{w}_{k, i} & =\sum_{\ell=1}^K a_{\ell k} \bm{\psi}_{\ell, i} \label{eqn:diff2}
	\end{align}
	\label{diffusion}
\end{subequations}
Here, $\widehat{\nabla J}_k(\cdot)$ denotes a stochastic approximation of the true local gradient $\nabla J_k(\cdot)$ based on data available at time \( i \). A common choice is \( \widehat{\nabla J}_k(\cdot) \triangleq \nabla Q_k(\cdot; \boldsymbol{x}_i, \boldsymbol{y}_i) \), although many variants such as mini-batch approximations are possible. The scalars $a_{\ell k}$ are combination weights satisfying:
\begin{equation}
	a_{\ell k} \geq 0, \quad \sum_{\ell \in \mathcal{N}_k} a_{\ell k}=1, \quad a_{\ell k}=0 \text { if } \ell \notin \mathcal{N}_k
\end{equation}
where $\mathcal{N}_k$ denotes the neighbourhood of agent $k$. Step (\ref{eqn:diff1}) corresponds to a local gradient descent step, encouraging (local) optimality of the intermediate weight vectors \( \boldsymbol{\psi}_{k, i} \), while (\ref{eqn:diff2}) corresponds to a consensus averaging step, driving local weight vectors towards a common solution. Convergence of (\ref{eqn:diff1})--(\ref{eqn:diff2}) to a minimizer, first- or second order-stationary point of \eqref{eqn:aggObjec} has been established under various conditions on the local objectives \(J_k(\cdot) \) \cite{Nedic:2009, Chen:2015, Lian:2017,Vlaski21P2}. In all of these cases, an important quantity in the rate of convergence is the mixing rate of \( A = [a_{\ell k}] \), given by its second largest eigenvalue \( \lambda_2 \). As discussed in \cite{Lin:2004}, an effective albeit suboptimal method for achieving fast convergence is the Metropolis rule:
\begin{equation}
    \label{eqn:metropolis}
    a_{\ell k}= \begin{cases}1 / \max \left\{n_k, n_{\ell}\right\}, & \text { if } k \neq \ell \text { and }\ell \in \mathcal{N}_k \\ 1-\sum_{\ell \in \mathcal{N}_k \backslash\{k\}} a_{\ell k}, & \text { when } k=\ell \\ 0, & \text { otherwise }\end{cases}
\end{equation}
where $n_k\triangleq|\mathcal{N}_k|$ is the degree of node $k$. 

Observe that the combination step (\ref{eqn:diff2}) drives the parameters \( \boldsymbol{w}_{k, i} \) towards consensus. Other strategies for decentralized optimization and learning, such as EXTRA \cite{Shi:2015} or Exact diffusion \cite{Yuan:2019} rely on primal-dual constructions, while gradient-tracking based algorithms such as NEXT \cite{Lorenzo:2016} or DIGing \cite{Nedic:2017a} rely on the dynamic consensus mechanism. All of them, however, in a manner analogous to (\ref{eqn:diff2}) encourage equality in the parameter space.

In this work, we deviate from this paradigm and argue that when training neural networks, it may be more appropriate to encourage consensus in the function space. This is because modern neural networks are highly over-parameterized, and as a result the same function may be obtained with highly variable parameterizations. Indeed, some evidence exists in the literature that allowing for some variability in local parameterizations during training can aid generalization performance of deep neural networks~\cite{Kong:2021}. Motivated by these observations, we formulate a penalty-based variant of (\ref{eqn:aggObjec}), where agents are allowed to maintain distinct local parameterizations, provided that the resulting function is approximately the same. We provide analytical convergence guarantees for the resulting algorithm, and demonstrate numerically its benefit in deep learning.


\section{DRT Diffusion}
It can be verified that the diffusion algorithm~\eqref{eqn:diff1}--\eqref{eqn:diff2} is equivalent to incremental stochastic gradient descent applied to the following penalized variant of~\eqref{eqn:aggObjec}~\cite{Yuan:2019}:
\begin{align}
    \min_{\{w_k\}_{k=1}^K}J_k(w_k)+\frac{\eta}{2}\sum_{\ell \in\mathcal{N}_k}c_{\ell k}{\lVert w_k - w_{\ell}\lVert^2}
    \label{eqn:orig_obj}
\end{align}
where \( c_{\ell k} > 0 \) if and only if \( \ell \in \mathcal{N}_k \). We can observe that under this relaxation, individual agents are allowed to maintain local models \( w_{k} \), but are nevertheless encouraged to ensure \( w_{k} \approx w_{\ell} \) through the penalty term \( \frac{\eta}{2}\sum_{\ell \in\mathcal{N}_k}c_{\ell k}{\lVert w_k - w_{\ell}\lVert^2} \). In this work, instead of penalizing difference between parameters, we propose to penalize the difference between neural network outputs:
\begin{align}
    \min_{\{w_k\}_{k=1}^K}J_k(w_k)+\frac{\eta}{2}\sum_{\ell \in\mathcal{N}_k}c_{\ell k}\frac{\lVert f(x; w_k) - f(x; w_{\ell})\lVert^2}{\lVert f(x; w_{\ell}) \lVert^2}
\end{align}
One challenge of this formulation is that differentiating the penalty term \( \frac{\eta}{2}\sum_{\ell \in\mathcal{N}_k}c_{\ell k}\frac{\lVert f(x; w_k) - f(x; w_{\ell})\lVert^2}{\lVert f(x; w_{\ell}) \lVert^2} \) is computationally prohibitive. Instead, we will replace this penalty by an upper bound, motivated by the Deep Relative Trust (DRT) distance measure~\cite{Bernstein2020}:
\begin{equation}\label{eq:berny}
    \frac{\lVert f(x; w_{\ell}) - f(x; w_k)\lVert}{\lVert f(x; w_k) \lVert }\leq\prod_{p=1}^{L} \left(1+\frac{\lVert w_{k}^{(p)} - w_{\ell}^{(p)}\lVert}{\lVert w_{k}^{(p)}\lVert}\right)-1
\end{equation}
Using arguments analogous to those in~\cite{Bernstein2020}, we can establish a quadratic variant:
\begin{equation}
    \begin{aligned}\label{eq:relative_trust}
        & \frac{\lVert f(x; w_k) - f(x; w_{\ell})\lVert^2}{\lVert f(x; w_{\ell}) \lVert^2} \\ 
        \leq & 2^{L+1} \prod_{p=1}^{L} \left(1+\frac{\lVert w_{k}^{(p)} - w_{\ell}^{(p)}\lVert^2}{\lVert w_{\ell}^{(p)}\lVert^2}\right) + 2
    \end{aligned}
\end{equation}
which can be more conveniently differentiated than~\eqref{eq:berny}. This gives rise to the penalized optimization problem:
\begin{equation}
    \label{eqn:obj_func}
    \begin{aligned}
        & \min_{\{w_k\}_{k=1}^K}J_k(w_k)\\
        + & \frac{\eta}{2}\sum_{\ell \in\mathcal{N}_k}c_{\ell k}\left[2^{L+1}\prod_{p=1}^{L} \left(1+\frac{\lVert w_{k}^{(p)} - w_{\ell}^{(p)}\lVert^2}{\lVert w_{\ell}^{(p)}\lVert^2 + \kappa}\right)+2\right]
    \end{aligned}
\end{equation}
where we added \( \kappa \ge 0 \) for numerical stability. The parameter \( \eta \) controls the penalty associated with differences in neural network outputs, and is chosen as \(\eta = \frac{1}{\mu}\), where \(\mu\) represents the learning rate for neural network training.

When we compare (\ref{eqn:obj_func}) with (\ref{eqn:orig_obj}), we observe that as a result of the relaxation~\eqref{eq:relative_trust}, the regularized problem~\eqref{eqn:obj_func} still ensures consensus in the parameter space as $\eta \rightarrow \infty$. For finite \( \eta\), however, the optimal solutions to~\eqref{eqn:orig_obj} and the proposed formulation~\eqref{eqn:obj_func} will differ, with the penalty of the latter being driven in proportion to the effect individual weights have on the output of the neural network.

After applying an incremental stochastic gradient construction to \eqref{eqn:obj_func} and subsequently normalizing the combination weights, we obtain the DRT diffusion algorithm:
\begin{equation}
    \begin{aligned}
    & \bm{\psi}_{k, i}=\bm{w}_{k, i-1}-\mu \widehat{\nabla J}_k\left(\bm{w}_{k, i-1}\right) \\
    & \bm{w}_{k, i}^{(p)}=\sum_{\ell \in \mathcal{N}_k} \bm{a}_{\ell k, i}^{(p)} \bm{\psi}_{\ell, i}^{(p)}
    \end{aligned}
\end{equation}
Observe that the resulting algorithm resembles the classical diffusion strategy~\eqref{eqn:diff1}--\eqref{eqn:diff2}, except that the combination weights as a result of the penalty in~\eqref{eqn:obj_func} now also depend on the layer index \( p \) and time \( i \). More specifically, the combination weights $\bm{a}_{\ell k, i}^{(p)}$ are given by:
\begin{equation}
    \bm{a}_{\ell k, i}^{(p)}= \frac{\tilde{\bm{a}}_{\ell k, i}^{(p)}}{\sum_{\ell \in \mathcal{N}_k} \tilde{\bm{a}}_{\ell k, i}^{(p)}}
\end{equation}
where
\begin{equation}
    \tilde{\bm{a}}_{\ell k, i}^{(p^{*})}= \begin{cases} \min\left(\dbtilde{\bm{a}}_{\ell k, i}^{(p^{*})}, N\min_{\ell}\left(\dbtilde{\bm{a}}_{\ell k, i}^{(p^{*})+}\right)\right) & \ell \neq k \\ \frac{c_{kk}}{n_k-1}\sum_{\ell \in \mathcal{N}_k \backslash \{k\}} \tilde{\bm{a}}_{\ell k, i}^{(p^{*})}  & \ell =k\end{cases} 
    \label{eqn:matrixCons}
\end{equation}
\begin{equation}
    \dbtilde{\bm{a}}_{\ell k, i}^{(p^{*})}= \begin{cases} c_{\ell k}\frac{2^{L+1}\prod_{p=1}^{P} (1+\frac{\lVert \bm{w}_{k,i}^{(p)} - \bm{w}_{\ell, i}^{(p)}\lVert^2}{\lVert \bm{w}_{\ell, i}^{(p)}\lVert^2+\kappa})}{ \lVert \bm{w}_{\ell, i}^{(p^{*})}\lVert^2+\lVert \bm{w}_{k,i}^{(p^{*})} - \bm{w}_{\ell,i}^{(p^{*})}\lVert^2 + \kappa} & \ell \neq k \\ 0  & \ell =k\end{cases} 
\end{equation}

Here, \( n_k \) represents the number of neighbors associated with agent \( k \). The parameter \( N \), which satisfies \( N \geq 1 \), determines the smallest possible positive entry in the mixing matrices. This minimum entry is given by \( \frac{1}{(K-1)N + 1} \), where \( K \) denotes the total number of agents. Furthermore, \( \min_{\ell}\left(\dbtilde{\bm{a}}_{\ell k, i}^{(p^{*})+}\right) \) refers to the smallest positive value of \( \dbtilde{\bm{a}}_{\ell k, i}^{(p^{*})} \) among all indices \( \ell \). It can be verified that under this construction:
\begin{equation}
    \bm{a}_{\ell k, i}^{(p)} \geq 0, \quad \sum_{\ell \in \mathcal{N}_k} \bm{a}_{\ell k, i}^{(p)}=1, \quad \bm{a}_{\ell k, i}^{(p)}=0 \text { if } \ell \notin \mathcal{N}_k
\end{equation}
The effect of these combination weights is that parameter deviations are penalized in relation to their effect on the output of the neural network.

\section{Convergence Analysis}
We first need to relate the connectivity of the underlying graph topology captured in $C = [c_{\ell k}]$ to the connectivity of the induced combination matrices \( \bm{A}_{i}^{(p)} \triangleq \left[\bm{a}_{\ell k, i}^{(p)}\right] \).

\begin{assumption}[Strongly-connected graph]\label{assp: sc}
    The graph described by the weight matrix $C = [c_{\ell k}]$ is strongly-connected, implying that \( C \) is a primitive matrix.
    \qed
\end{assumption}



\begin{lemma}[Graph-compatible $\bm{A}_{i}^{(p)}$]\label{lem:graph_compatible}
    Under Assumption \ref{assp: sc} and the construction \eqref{eqn:matrixCons}, the graph represented by the weighted combination matrix \( \bm{A}_{i}^{(p)} \triangleq \left[\bm{a}_{\ell k, i}^{(p)}\right] \) is compatible with the graph described by \( C \) for all \( p \) and all \( i \) in the sense that:
    \begin{equation}
        \label{eqn:graph_compatible}
        \bm{a}_{\ell k, i}^{(p)} > 0 \text{ for all } \ell \in \mathcal{N}_k, \qquad \bm{a}_{\ell k, i}^{(p)} = 0 \text{ for all } \ell \notin \mathcal{N}_k.
    \end{equation}
    Moreover, for all \( p \) and all \( i \), the nonzero elements in the mixing matrices are lower bounded as follows:
    \begin{equation}
        \label{eqn:lower_bound}
        \bm{a}_{\ell k, i}^{(p)} > \frac{1}{(K-1)N + 1} \text{ whenever } \bm{a}_{\ell k, i}^{(p)} > 0
    \end{equation}
    \qed
\end{lemma}
\begin{proof}
    From Equation \eqref{eqn:matrixCons}, it follows that for any \( \ell \) and \( k \), \( \bm{a}_{\ell k, i}^{(p)} = 0 \) if and only if \( c_{\ell k} = 0 \). This establishes that the graph represented by \( \bm{A}_{i}^{(p)} \) is compatible with the graph described by \( C \), as stated in Equation \eqref{eqn:graph_compatible}. 
    
    Furthermore, the construction in Equation \eqref{eqn:matrixCons} ensures that the smallest positive element in the mixing matrices is given by $\frac{1}{(K-1)N + 1}$, which occurs when a column contains one unique smallest entry, and all other entries in the same column are \( N \) times this smallest entry.
\end{proof}

To account for the stochastic gradient dynamics, we employ the following assumptions on the gradients and their stochastic approximations.

\begin{assumption}[Lipschitz gradients] \label{assup:lip}
    For each $k$, the gradient $\nabla J_k(\cdot)$ is Lipschitz, namely, for any $x, y \in \mathbb{R}^M$:
    \begin{equation}
        \left\|\nabla J_k(x)-\nabla J_k(y)\right\| \leq \delta\|x-y\|
    \end{equation}
    \qed
\end{assumption}

\begin{assumption}[Bounded gradients]\label{asup:bg}
    For each $k$, the gradient $\nabla J_k(\cdot)$ is bounded, namely, for any $x \in \mathbb{R}^M$ :
    \begin{equation}
        \left\|\nabla J_k(x)\right\| \leq G
        \label{eqn:bg}
    \end{equation}
    \qed
\end{assumption}

\begin{assumption}[Gradient noise process]\label{asup:gn[]}
    For each $k$, the gradient noise process is defined as
    \begin{equation}
    \boldsymbol{s}_{k, i}\left(\boldsymbol{w}_{k, i-1}\right)=\widehat{\nabla J}_k\left(\boldsymbol{w}_{k, i-1}\right)-\nabla J_k\left(\boldsymbol{w}_{k, i-1}\right)
    \end{equation}
    and satisfies
    \begin{equation}
    \label{eqn:gd}
    \begin{aligned}
    \mathbb{E}\left[\bm{s}_{k, i}\left(\boldsymbol{w}_{k, i-1}\right) \mid \boldsymbol{w}_{k, i-1}\right] & =0 \\
    \mathbb{E}\left[\left\|\boldsymbol{s}_{k, i}\left(\boldsymbol{w}_{k, i-1}\right)\right\|^2 \mid \boldsymbol{w}_{k, i-1}\right] & \leq \sigma_k^2
    \end{aligned}
    \end{equation}
    for some non-negative constants $\sigma_k^2$. We also assume that the gradient noise processes are pairwise uncorrelated, i.e.:
    \begin{align}
        \mathbb{E}\left\{\boldsymbol{s}_{k, i}\left(\boldsymbol{w}_{k, i-1}\right) \boldsymbol{s}_{\ell, i}\left(\boldsymbol{w}_{\ell, i-1}\right)^{\top} \mid \bm{w}_{k,i-1}, \bm{w}_{\ell,i-1}\right\} = 0
        \label{eqn:uncorrelated}
    \end{align}
\qed
\end{assumption}

\subsection{Network Centroid}
To account for the time-varying characteristics of the mixing matrices in the proposed algorithm, we introduce the following lemma, which is a minor variation of Lemma 5.2.1 in~\cite{Tsitsiklis84}.
\begin{lemma}[Time-varying weight vector~\cite{Tsitsiklis84}]
    Under Assumption \ref{assp: sc} and construction \eqref{eqn:matrixCons}, there exists a sequence of random weight vectors \( \{\bm{\phi}_i\} \) such that for $t \ge i$:  
    \begin{align}
    \label{eqn:phi_definition}
    \left\|\bm{A}_t^{\top} \bm{A}_{t-1}^{\top} \cdots \bm{A}_i^{\top} - \mathds{1} \bm{\phi}_i^{\top}\right\|^2 \leq C \xi^{t-i},
    \end{align}  
    where the constants satisfy \( C > 0 \) and \( \xi \in (0, 1) \).   
    
    Additionally, the sequence satisfies the following recursive relationship:  
    \begin{align}
    \label{eqn:phi_A}
    \bm{\phi}_i = \bm{A}_i \bm{\phi}_{i+1}.
    \end{align}  
    \qed
\end{lemma}
\noindent We can then define the network centroid as:
\begin{align}
   \bm{w}_{c, i} \triangleq \sum_{k=1}^K [\bm{\phi}_{i}]_k \bm{w}_{k, i}
\end{align}
where \( [\bm{\phi}_{i}]_k \) denotes the \( k \)-th entry of the time-varying weight vector in \eqref{eqn:phi_definition}.

\subsection{Network Disagreement and Descent}
\begin{lemma}[Network Disagreement] Under Assumptions \ref{assp: sc}-\ref{asup:gn[]} and for sufficiently small step-sizes $\mu$, the network disagreement is bounded as:\label{lem:disg}
 \begin{equation}
    \label{eqn:disg2nd}
         \sum_{k=1}^K \mathbb{E}\left\|\bm{w}_{k, i} - \bm{w}_{c,i} \right\|^2 \leq \mu^2 \frac{4C}{\left(1-\sqrt{\xi}\right)^2} K \left(G^2+\max_k \sigma_k^2\right)
\end{equation}
\qed
\end{lemma}

\begin{proof}

    The argument uses techniques from~\cite{Vlaski21P1} to establish clustering. It is omitted due to space limitations.
\end{proof}
Having established clustering of the agents around a suitably chosen reference point, we can now establish convergence to a first-order stationarity. For simplicity, for this theorem only, we restrict ourselves to the setting of IID data, meaning that agents observe data that is identically distributed. Formally:
\begin{assumption}[Common Objective Functions]\label{assp:iid} The local data distributions are identical, and hence \( J_k(w) = J_{\ell}(w)\) for all \( k \) and \( \ell \).
\end{assumption}
\begin{theorem}[Descent Relation]\label{thm:descent}
    Under Assumptions \ref{assp: sc}-\ref{assp:iid} and for sufficiently small step-sizes, we have:
    \begin{equation}
        \label{eqn:descentRel}
        \begin{aligned}
            & \mathbb{E}\left\{J\left(\boldsymbol{w}_{c, i}\right) \mid \boldsymbol{w}_{c, i-1}\right\} 
            \leq J\left(\boldsymbol{w}_{c, i-1}\right) 
        \end{aligned}
    \end{equation}
    whenever
    \begin{equation}
    \left\|\nabla J\left(\boldsymbol{w}_{c, i-1}\right)\right\|^2 \geq \mu \frac{c_2}{c_1} \left(1+\frac{1}{\pi}\right)
    \end{equation}
    where \( c_1  \) and \( c_2 \) are constants independent of \( \mu \), and $0 < \pi < 1$ is a parameter to be chosen.
    \qed
\end{theorem}
\begin{proof}
    The argument is analogous to~\cite{Vlaski21P2} and omitted due to space limitations.
\end{proof}

This theorem ensures the proposed algorithm is a decent recursion for the network centroid in the IID scenario whenever the network centroid $\bm{w}_{c,i}$ is not $O(\mu)$-stationary. Then from Crollollary 1 in \cite{nonconvex}, we can conclude that at some time $i^{\star} \leq O\left(1 / \mu^2\right)$, the network centroid will reach an $O(\mu)$-mean-square-stationary point.

\section{Simulation Results}
\label{sec:pagestyle}

\subsection{Simulation Setup}
Since the convergence guarantees are proven in the IID case, we aim to explore numerically whether DRT diffusion performs well in the non-IID case. For this purpose, we utilize the CIFAR-10 dataset for an image classification task. We consider a network of \( 16 \) agents, each equipped with a non-IID local dataset sampled from CIFAR-10 without replacement. To ensure the local datasets are non-IID, we first randomly select the number of image classes present at each agent, ranging between \( 5 \) and \( 8 \). Then, the total number of samples for each agent is randomly chosen to be between \( 1500 \) and \( 2000 \).
    
    To explore the performance of the DRT diffusion algorithm across different network topologies, we experiment with a ring, Erdős–Rényi model with the edge probability of \( 0.1 \), and Hypercube.
    
    We compare the performance of the DRT diffusion algorithm with the classical diffusion algorithm~\cite{Sayed:2014} using the optimal mixing matrix as described in \cite{Lin:2004}. Both algorithms are tested by locally training ResNet-20s for one epoch with a batch size of 128. After local training, the algorithms perform 3 consensus steps. The primary purpose of implementing consecutive consensus steps is to reduce the number of epochs needed to reach a steady state \cite{Kong:2021}.For the mixing matrix construction, $N$ is chosen to be $2K$.

\subsection{Results and Discussion}
        We compare performance based on steady-state performance, convergence rate and the generalization gap, which measures the difference between training and test accuracy across all experiments. In particular, we examine the feature of DRT diffusion that facilitates distinct local parameterizations.

        \begin{table}[htb]
    	\centering
    	\caption{Steady-state test accuracy with different topologies}
    	{\renewcommand{\arraystretch}{1}
    		\begin{tabular}{|c|c|c|c|}
    			\hline
    			Topology & $\lambda_2$ & Classical diffusion & DRT diffusion \\
                    
                    \hline
                    Ring & 0.949 & 66.21\% & 69.60\% \\
                    \hline
                    Erdős–Rényi & 0.905 & 65.55\% & 69.22\% \\
                    \hline
                    Hypercube & 0.600 & 71.05\% & 72.00\% \\
                    \hline
                    
    		\end{tabular}
    	}
    	\label{tbl:ssci}
    \end{table}
    
    \begin{figure}[htb]
    \begin{minipage}[b]{.48\linewidth}
      \centering
      \centerline{\includegraphics[scale=.28]{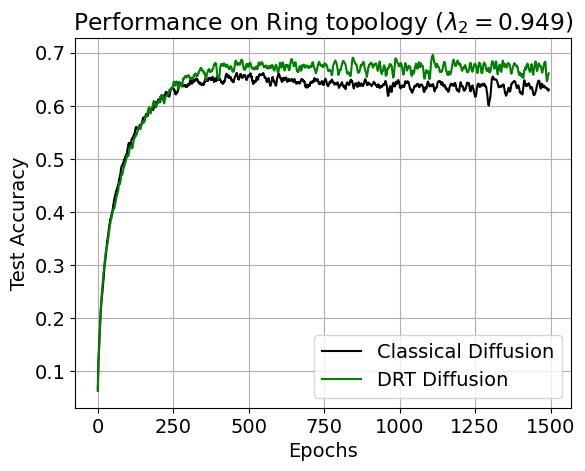}}
      \centerline{(a) Ring $\lambda_2=0.949$}\medskip
    \end{minipage}
    \begin{minipage}[b]{.48\linewidth}
      \centering
      \centerline{\includegraphics[scale=.28]{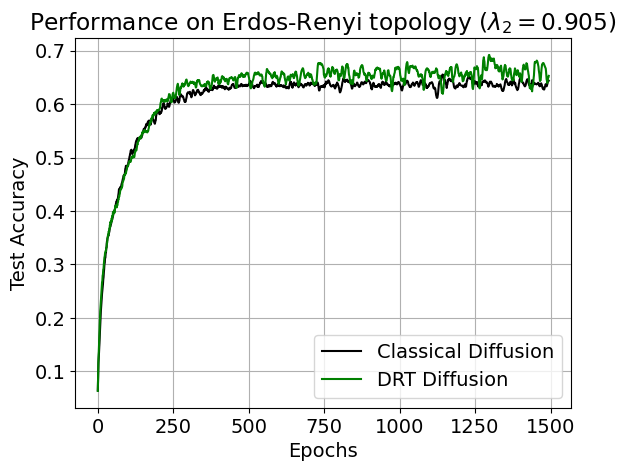}}
      \centerline{(b) Erdős–Rényi $\lambda_2=0.905$}\medskip
    \end{minipage}
    \centering
    \begin{minipage}[b]{.48\linewidth}
      \centering
      \centerline{\includegraphics[scale=.28]{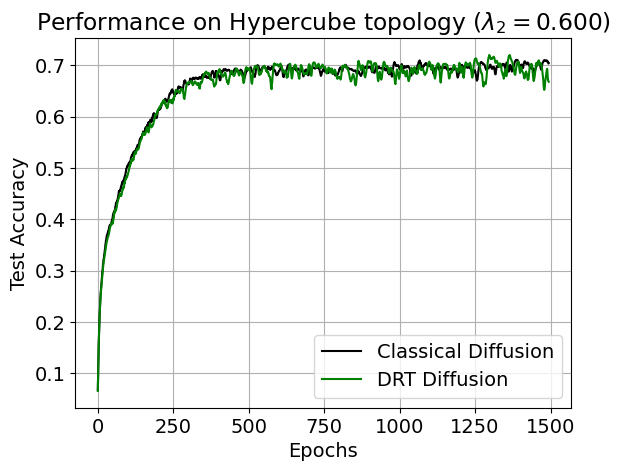}}
      \centerline{(c) Hypercube $\lambda_2=0.600$}\medskip
    \end{minipage}
    \caption{Learning curves for a decentralized network with 16 agents, employing ResNet-20 on CIFAR-10 with non-IID data at each agent}
    \label{fig:acc}
    \end{figure}
        
        \subsubsection{Steady-state Performance}
        The steady-state test accuracies are displayed in Table \ref{tbl:ssci}, and the learning curves along with generalization gaps are depicted in Figures \ref{fig:acc} and \ref{fig:cd}. The results in Table \ref{tbl:ssci} show that DRT diffusion surpasses classical diffusion in steady-state performance, especially in sparse topologies. However, for the well-connected Hypercube topology, the performance difference is minimal. Interestingly, DRT diffusion exhibits greater fluctuations than classical diffusion in the Hypercube, suggesting it may not be the optimal choice for well-connected topologies. These findings indicate that DRT diffusion, by promoting consensus in the function space, facilitates efficient information exchange among agents.
        
        \subsubsection{Convergence Rate}
        Figure \ref{fig:acc} shows that DRT diffusion achieves the same convergence rate as classical diffusion with the fast-mixing Metropolis rule. This is notable since the Metropolis rule is constructed to result in fast information diffusion, while the proposed DRT-based scheme is only designed to ensure equality in neural network outputs.

        \subsubsection{Generalization Gap}
        The generalization gap indicates how well the model performs on unseen data. Figure \ref{fig:cd} demonstrates that DRT diffusion has a smaller generalization gap compared to classical diffusion with sparse topologies. This suggests that the proposed algorithms has an implicit bias towards minimizers that generalize better, warranting further investigation along the lines of~\cite{Zhu:2022, Zhu:2023}.
        
    \begin{figure}[htb]
    \begin{minipage}[b]{.48\linewidth}
      \centering
      \centerline{\includegraphics[scale=.28]{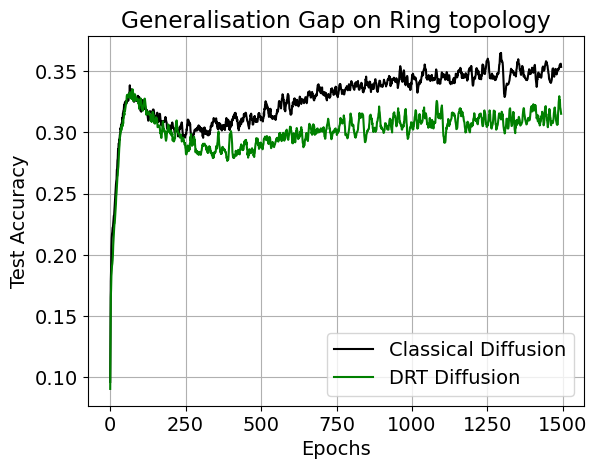}}
      \centerline{(a) Ring $\lambda_2=0.949$}\medskip
    \end{minipage}
    \begin{minipage}[b]{.48\linewidth}
      \centering
      \centerline{\includegraphics[scale=.28]{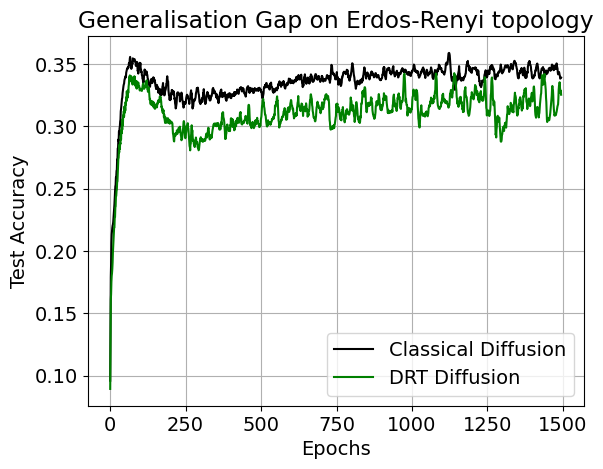}}
      \centerline{(b) Erdős–Rényi $\lambda_2=0.905$}\medskip
    \end{minipage}
    \centering
    \begin{minipage}[b]{.48\linewidth}
      \centering
      \centerline{\includegraphics[scale=.28]{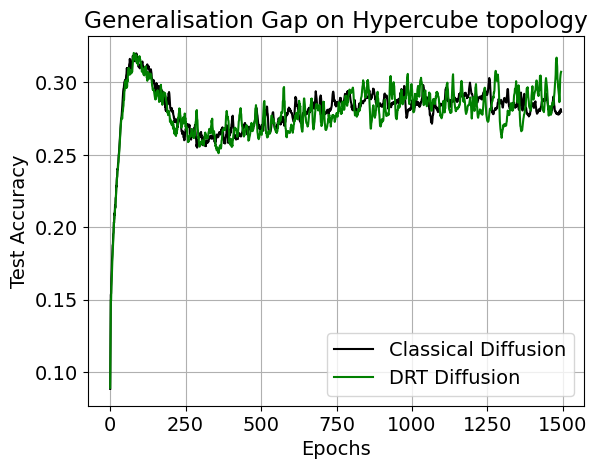}}
      \centerline{(c) Hypercube $\lambda_2=0.600$}\medskip
    \end{minipage}
    \caption{Generalization gap for a decentralized network with 16 agents, employing ResNet-20 on CIFAR-10 with non-IID data at each agent}
    \label{fig:cd}
    \end{figure}

\bibliographystyle{IEEEtran}
\bibliography{refs}

\clearpage

\appendix
\subsection{Proof of Lemma \ref{lem:disg}}
The DRT diffusion recursion for a specific layer of the neural network can be compactly expressed as:
\begin{equation}
    \bcw_{i} = \bm{\mathcal{A}}_{i-1}^{\top} \left( \bcw_{i-1} - \mu \widehat{\nabla \mathcal{J}} \left( \bcw_{i-1} \right) \right)
    \label{eqn:DRTlayer}
\end{equation}
To begin with, we study more closely the evolution of the individual estimates $\bm{w}_{k, i}$ relative to the network centroid $\bm{w}_{c, i}$. Multiplying (\ref{eqn:DRTlayer}) by $\left[I-\left(\mathds{1} \bm{\phi}_i^{\top} \otimes I\right)\right]$ from the left yields in light of (\ref{eqn:phi_A}):
\begin{equation}
    \begin{aligned}
        \label{eqn:disagEvo}
        & \left[I-\left(\mathds{1} \bm{\phi}_i^{\top} \otimes I\right)\right]\bcw_i \\
        = & \left[I-\left(\mathds{1} \bm{\phi}_i^{\top} \otimes I\right)\right] \bm{\mathcal{A}}_{i-1}^{\top} \left( \bcw_{i-1} - \mu \widehat{\nabla \mathcal{J}} \left( \bcw_{i-1} \right) \right) \\
        \stackrel{(\ref{eqn:phi_A})}{=} & \left(\bm{\mathcal{A}}_{i-1}^{\top}-\mathds{1}\bm{\phi}_{i-1}^{\top}\otimes I\right)  \left(\bcw_{i-1} - \mu  \widehat{\nabla \mathcal{J}}\left(\bcw_{i-1}\right)\right)
    \end{aligned}
\end{equation}

Then, we can obtain the recursion for the deviation from the network centroid:
\begin{equation}
    \label{eqn:disagRecur}
    \begin{aligned}
        & \bcw_i - \bcw_{c,i} \\
        = & \left[I-\left(\mathds{1} \bm{\phi}_i^{\top} \otimes I\right)\right]\bcw_i \\
        \stackrel{(\ref{eqn:disagEvo})}{=} & \left(\bm{\mathcal{A}}_{i-1}^{\top}-\mathds{1}\bm{\phi}_{i-1}^{\top}\otimes I\right)  \left(\bcw_{i-1} - \mu  \widehat{\nabla \mathcal{J}}\left(\bcw_{i-1}\right)\right) \\
        \stackrel{(a)}{=} & \left(\bm{\mathcal{A}}_{i-1}^{\top}-\mathds{1}\bm{\phi}_{i-1}^{\top}\otimes I\right) \left[I-\left(\mathds{1} \bm{\phi}_{i-1}^{\top} \otimes I\right)\right] \\
        & \left(\bcw_{i-1} - \mu  \widehat{\nabla \mathcal{J}}\left(\bcw_{i-1}\right)\right)  \\
        = & \left(\bm{\mathcal{A}}_{i-1}^{\top}-\mathds{1}\bm{\phi}_{i-1}^{\top}\otimes I\right) \left[\left(\bcw_{i-1} - \bcw_{c,i-1}\right)\right.\\
        & \left. -\mu  \widehat{\nabla \mathcal{J}}\left(\bcw_{i-1}\right)\right]
    \end{aligned}
\end{equation}
 where (a) follows from the equality:
 \begin{equation}
 \begin{aligned}
      & \left(\bm{\mathcal{A}}_{i-1}^{\top}-\mathds{1}\bm{\phi}_{i-1}^{\top}\otimes I\right) \left[I-\left(\mathds{1} \bm{\phi}_{i-1}^{\top} \otimes I\right)\right] \\
      = & \bm{\mathcal{A}}_{i-1}^{\top} - \bm{\mathcal{A}}_{i-1}^{\top}\mathds{1} \bm{\phi}_{i-1}^{\top} \otimes I - \mathds{1}\bm{\phi}_{i-1}^{\top}\otimes I \\
      & + \mathds{1}\left(\bm{\phi}_{i-1}^{\top}\mathds{1}\right)\bm{\phi}_{i-1}^{\top}\otimes I\\
      = & \bm{\mathcal{A}}_{i-1}^{\top} - \mathds{1} \bm{\phi}_{i-1}^{\top} \otimes I - \mathds{1}\bm{\phi}_{i-1}^{\top}\otimes I + \mathds{1}\bm{\phi}_{i-1}^{\top}\otimes I\\
      = & \bm{\mathcal{A}}_{i-1}^{\top}-\mathds{1}\bm{\phi}_{i-1}^{\top}\otimes I
 \end{aligned}
 \end{equation}

 We study the disagreement recursion further:
 \begin{equation}
    \label{eqn:disag_to_zero}
    \begin{aligned}
        & \bcw_i - \bcw_{c,i} \\
        = & \left(\bm{\mathcal{A}}_{i-1}^{\top}-\mathds{1}\bm{\phi}_{i-1}^{\top}\otimes I\right) \left[\left(\bcw_{i-1} - \bcw_{c,i-1}\right) \right. \\
        & \left.- \mu  \widehat{\nabla \mathcal{J}}\left(\bcw_{i-1}\right)\right] \\
        \stackrel{\eqref{eqn:disagRecur}}{=} & \left(\bm{\mathcal{A}}_{i-1}^{\top}-\mathds{1}\bm{\phi}_{i-1}^{\top}\otimes I\right) \left(\bm{\mathcal{A}}_{i-2}^{\top}-\mathds{1}\bm{\phi}_{i-2}^{\top}\otimes I\right) \\
        &\left(\bcw_{i-2} - \bcw_{c,i-2}\right) \\
        & - \mu \left(\bm{\mathcal{A}}_{i-1}^{\top}-\mathds{1}\bm{\phi}_{i-1}^{\top}\otimes I\right) \left(\bm{\mathcal{A}}_{i-2}^{\top}-\mathds{1}\bm{\phi}_{i-2}^{\top}\otimes I\right) \\
        & \widehat{\nabla \mathcal{J}}\left(\bcw_{i-2}\right) \\
        & - \mu \left(\bm{\mathcal{A}}_{i-1}^{\top}-\mathds{1}\bm{\phi}_{i-1}^{\top}\otimes I\right) \widehat{\nabla \mathcal{J}}\left(\bcw_{i-1}\right) \\
        \stackrel{(\ref{eqn:phi_A})}{=} & \left(\bm{\mathcal{A}}_{i-1}^{\top}\bm{\mathcal{A}}_{i-2}^{\top}-\mathds{1}\bm{\phi}_{i-2}^{\top}\otimes I\right) \left(\bcw_{i-2} - \bcw_{c,i-2}\right) \\
        & - \mu \left(\bm{\mathcal{A}}_{i-1}^{\top}\bm{\mathcal{A}}_{i-2}^{\top}-\mathds{1}\bm{\phi}_{i-2}^{\top}\otimes I\right) \widehat{\nabla \mathcal{J}}\left(\bcw_{i-2}\right) \\
        & - \mu \left(\bm{\mathcal{A}}_{i-1}^{\top}-\mathds{1}\bm{\phi}_{i-1}^{\top}\otimes I\right) \widehat{\nabla \mathcal{J}}\left(\bcw_{i-1}\right) \\
        \stackrel{\eqref{eqn:disagRecur}}{=} & \left(\bm{\mathcal{A}}_{i-1}^{\top}\bm{\mathcal{A}}_{i-2}^{\top}-\mathds{1}\bm{\phi}_{i-2}^{\top}\otimes I\right) \left(\bm{\mathcal{A}}_{i-3}^{\top}-\mathds{1}\bm{\phi}_{i-3}^{\top}\otimes I\right) \\
        & \left(\bcw_{i-3} - \bcw_{c,i-3}\right) \\
        & - \mu \left(\bm{\mathcal{A}}_{i-1}^{\top}\bm{\mathcal{A}}_{i-2}^{\top}-\mathds{1}\bm{\phi}_{i-2}^{\top}\otimes I\right) \left(\bm{\mathcal{A}}_{i-3}^{\top}-\mathds{1}\bm{\phi}_{i-3}^{\top}\otimes I\right)\\
        & \widehat{\nabla \mathcal{J}}\left(\bcw_{i-3}\right) \\
        & - \mu \left(\bm{\mathcal{A}}_{i-1}^{\top}\bm{\mathcal{A}}_{i-2}^{\top}-\mathds{1}\bm{\phi}_{i-2}^{\top}\otimes I\right) \widehat{\nabla \mathcal{J}}\left(\bcw_{i-2}\right) \\
        & - \mu \left(\bm{\mathcal{A}}_{i-1}^{\top}-\mathds{1}\bm{\phi}_{i-1}^{\top}\otimes I\right) \widehat{\nabla \mathcal{J}}\left(\bcw_{i-1}\right) \\
        = & \left(\bm{\mathcal{A}}_{i-1}^{\top}\bm{\mathcal{A}}_{i-2}^{\top}\bm{\mathcal{A}}_{i-3}^{\top}-\mathds{1}\bm{\phi}_{i-3}^{\top}\otimes I\right) \left(\bcw_{i-3} - \bcw_{c,i-3}\right) \\
        & - \mu \left(\bm{\mathcal{A}}_{i-1}^{\top}\bm{\mathcal{A}}_{i-2}^{\top}\bm{\mathcal{A}}_{i-3}^{\top}-\mathds{1}\bm{\phi}_{i-3}^{\top}\otimes I\right) \widehat{\nabla \mathcal{J}}\left(\bcw_{i-3}\right) \\
        & - \mu \left(\bm{\mathcal{A}}_{i-1}^{\top}\bm{\mathcal{A}}_{i-2}^{\top}-\mathds{1}\bm{\phi}_{i-2}^{\top}\otimes I\right) \widehat{\nabla \mathcal{J}}\left(\bcw_{i-2}\right) \\
        & - \mu \left(\bm{\mathcal{A}}_{i-1}^{\top}-\mathds{1}\bm{\phi}_{i-1}^{\top}\otimes I\right) \widehat{\nabla \mathcal{J}}\left(\bcw_{i-1}\right) \\
        = & \left(\bm{\mathcal{A}}_{i-1}^{\top}\bm{\mathcal{A}}_{i-2}^{\top}\cdots\bm{\mathcal{A}}_{0}^{\top}-\mathds{1}\bm{\phi}_{0}^{\top}\otimes I\right) \left(\bcw_{0} - \bcw_{c,0}\right) \\
        & -\mu \sum_{n=1}^{i-1}\left(\bm{\mathcal{A}}_{i-1}^{\top}\bm{\mathcal{A}}_{i-2}^{\top}\cdots\bm{\mathcal{A}}_{n}^{\top}-\mathds{1}\bm{\phi}_{n}^{\top}\otimes I\right)\widehat{\nabla \mathcal{J}}\left(\bcw_{n}\right) 
    \end{aligned}
\end{equation}

We compute the norm and then take the expectation:
\begin{equation}
    \label{eqn:expt_disag1}
    \begin{aligned}
        & \mathbb{E} \left\|\bcw_i - \bcw_{c,i}\right\| \\
        = & \mathbb{E} \left\| \left(\bm{\mathcal{A}}_{i-1}^{\top}\bm{\mathcal{A}}_{i-2}^{\top}\cdots\bm{\mathcal{A}}_{0}^{\top}-\mathds{1}\bm{\phi}_{0}^{\top}\otimes I\right) \left(\bcw_{0} - \bcw_{c,0}\right) \right.\\
        & \left. -\mu \sum_{n=1}^{i-1}\left(\bm{\mathcal{A}}_{i-1}^{\top}\bm{\mathcal{A}}_{i-2}^{\top}\cdots\bm{\mathcal{A}}_{n}^{\top}-\mathds{1}\bm{\phi}_{n}^{\top}\otimes I\right)\widehat{\nabla \mathcal{J}}\left(\bcw_{n}\right) \right\| \\
        \leq & \mathbb{E} \left\| \left(\bm{\mathcal{A}}_{i-1}^{\top}\bm{\mathcal{A}}_{i-2}^{\top}\cdots\bm{\mathcal{A}}_{0}^{\top}-\mathds{1}\bm{\phi}_{0}^{\top}\otimes I\right) \left(\bcw_{0} - \bcw_{c,0}\right) \right\|\\
        & +\mu\mathbb{E}\left\| \sum_{n=1}^{i-1}\left(\bm{\mathcal{A}}_{i-1}^{\top}\bm{\mathcal{A}}_{i-2}^{\top}\cdots\bm{\mathcal{A}}_{n}^{\top}-\mathds{1}\bm{\phi}_{n}^{\top}\otimes I\right)\widehat{\nabla \mathcal{J}}\left(\bcw_{n}\right) \right\| \\
        \leq & \left\| \bm{\mathcal{A}}_{i-1}^{\top}\bm{\mathcal{A}}_{i-2}^{\top}\cdots\bm{\mathcal{A}}_{0}^{\top}-\mathds{1}\bm{\phi}_{0}^{\top}\otimes I\right\|\mathbb{E} \left\| \bcw_{0} - \bcw_{c,0} \right\|\\
        & +\mu\mathbb{E}\left\| \sum_{n=1}^{i-1}\left(\bm{\mathcal{A}}_{i-1}^{\top}\bm{\mathcal{A}}_{i-2}^{\top}\cdots\bm{\mathcal{A}}_{n}^{\top}-\mathds{1}\bm{\phi}_{n}^{\top}\otimes I\right)\widehat{\nabla \mathcal{J}}\left(\bcw_{n}\right) \right\| \\
        \stackrel{\eqref{eqn:phi_definition}}{\le} & \sqrt{C}\xi^{\frac{i-1}{2}}\mathbb{E} \left\| \bcw_{0} - \bcw_{c,0} \right\|\\
        & +\mu\mathbb{E}\left\| \sum_{n=1}^{i-1}\left(\bm{\mathcal{A}}_{i-1}^{\top}\bm{\mathcal{A}}_{i-2}^{\top}\cdots\bm{\mathcal{A}}_{n}^{\top}-\mathds{1}\bm{\phi}_{n}^{\top}\otimes I\right) \widehat{\nabla \mathcal{J}}\left(\bcw_{n}\right) \right\| \\
        \stackrel{(a)}{\le} & \sqrt{C}\xi^{\frac{i-1}{2}}\mathbb{E} \left\| \bcw_{0} - \bcw_{c,0} \right\|\\
        & +\mu\mathbb{E} \sum_{n=1}^{i-1}\left\|\left(\bm{\mathcal{A}}_{i-1}^{\top}\bm{\mathcal{A}}_{i-2}^{\top}\cdots\bm{\mathcal{A}}_{n}^{\top}-\mathds{1}\bm{\phi}_{n}^{\top}\otimes I\right)\widehat{\nabla \mathcal{J}}\left(\bcw_{n}\right) \right\| \\
        \le & \sqrt{C}\xi^{\frac{i-1}{2}}\mathbb{E} \left\| \bcw_{0} - \bcw_{c,0} \right\|\\
        & +\mu \sum_{n=1}^{i-1}\mathbb{E}\left\|\left(\bm{\mathcal{A}}_{i-1}^{\top}\bm{\mathcal{A}}_{i-2}^{\top}\cdots\bm{\mathcal{A}}_{n}^{\top}-\mathds{1}\bm{\phi}_{n}^{\top}\otimes I\right)\widehat{\nabla \mathcal{J}}\left(\bcw_{n}\right) \right\| \\
        \le & \sqrt{C}\xi^{\frac{i-1}{2}}\mathbb{E} \left\| \bcw_{0} - \bcw_{c,0} \right\|\\
        & +\mu \sum_{n=1}^{i-1}\left\|\bm{\mathcal{A}}_{i-1}^{\top}\bm{\mathcal{A}}_{i-2}^{\top}\cdots\bm{\mathcal{A}}_{n}^{\top}-\mathds{1}\bm{\phi}_{n}^{\top}\otimes I\right\| \mathbb{E}\left\|\widehat{\nabla \mathcal{J}}\left(\bcw_{n}\right) \right\| \\
        \stackrel{\eqref{eqn:phi_definition}}{\le} & \sqrt{C}\xi^{\frac{i-1}{2}}\mathbb{E} \left\| \bcw_{0} - \bcw_{c,0} \right\|\\
        & +\mu \sum_{n=1}^{i-1} \sqrt{C}\xi^{\frac{i-1-n}{2}} \mathbb{E}\left\|\widehat{\nabla \mathcal{J}}\left(\bcw_{n}\right) \right\| \\
    \end{aligned}
\end{equation}
where (a) follows from the general triangle inequality that $\left\|\sum_{k=1}^N x_k\right\| \le \sum_{k=1}^N \left\|x_k\right\|$. We examine the stochastic gradient term in greater detail. We have:
\begin{equation}
    \label{eqn:stochasticG1}
    \begin{aligned}
        & \left\|\widehat{\nabla \mathcal{J}}\left(\bcw_{i}\right)\right\| \\
        = & \left\| \nabla \mathcal{J}\left(\bcw_{i}\right)+ \operatorname{col}\left\{\bm{s}_{k, i}\left(\bm{w}_{k, i}\right)\right\}\right\| \\
        \leq & \left\|\nabla \mathcal{J}\left(\bcw_{i-1}\right)\right\|+\left\| \operatorname{col}\left\{\bm{s}_{k, i}\left(\bm{w}_{k, i-1}\right)\right\}\right\|
    \end{aligned}
\end{equation}
For the first term, we have:
\begin{equation}
    \label{eqn:network_level_gradient_bound}
    \begin{aligned}
        & \left\|\nabla \mathcal{J}\left(\bcw_{i-1}\right)\right\| \\
         \stackrel{(b)}{\leq} & \sqrt{\sum_{k=1}^K  \left\|\nabla J_k(\bm{w}_{k,i-1})\right\|^2}\\
         \stackrel{(\ref{eqn:bg})}{\leq} & \sqrt{\sum_{k=1}^K G^2} \leq \sqrt{K}G
    \end{aligned}
\end{equation}
where (b) expands $\sqrt{\|\cdot\|^2}$. For the gradient noise term we find under expectation:
\begin{equation}
    \label{eqn:network_level_gradient_noise_bound}
    \begin{aligned}
        & \mathbb{E}\left\|\operatorname{col}\left\{\bm{s}_{k, i}\left(\boldsymbol{w}_{k, i-1}\right)\right\}\right\| \\
        = & \mathbb{E}\sqrt{\sum_{k=1}^K\left\|\bm{s}_{k, i}\left(\boldsymbol{w}_{k, i-1}\right)\right\|^2} \\
        \stackrel{(a)}{\leq} & \sqrt{\mathbb{E} \sum_{k=1}^K\left\|\bm{s}_{k, i}\left(\boldsymbol{w}_{k, i-1}\right)\right\|^2} \\
        \le & \sqrt{ \sum_{k=1}^K \mathbb{E} \left\|\bm{s}_{k, i}\left(\boldsymbol{w}_{k, i-1}\right)\right\|^2} \\
        \stackrel{(\ref{eqn:gd})}{\leq} & \sqrt{K \max_k \sigma_k^2}
    \end{aligned}
\end{equation}
where (a) follows from the fact that $\sqrt{\cdot}^2$ is a concave function and Jensen's inequality. Substituting these relationships back into equation (\ref{eqn:expt_disag1}), we get:
\begin{equation}
    \label{eqn:expt_disg2}
    \begin{aligned}
        & \mathbb{E}\left\| \bcw_i - \bcw_{c,i}\right\| \\
        \le & \sqrt{C}\xi^{\frac{i-1}{2}}\mathbb{E} \left\| \bcw_{0} - \bcw_{c,0} \right\| +\mu \sum_{n=1}^{i-1} \sqrt{C}\xi^{\frac{i-1-n}{2}} \mathbb{E}\left\|\widehat{\nabla \mathcal{J}}\left(\bcw_{n}\right) \right\| \\
        \le & \sqrt{C}\xi^{\frac{i-1}{2}}\mathbb{E} \left\| \bcw_{0} - \bcw_{c,0} \right\| \\
        & +\mu \sum_{n=1}^{i-1} \sqrt{C}\xi^{\frac{i-1-n}{2}} \sqrt{K(G^2+\max_{k}\sigma_k^2)} \\
        \le & \sqrt{C}\xi^{\frac{i-1}{2}}\mathbb{E} \left\| \bcw_{0} - \bcw_{c,0} \right\| \\
        & +\mu \sum_{n=0}^{\infty} \sqrt{C}\xi^{\frac{n}{2}} \sqrt{K(G^2+\max_{k}\sigma_k^2)} \\
        \le & \sqrt{C}\xi^{\frac{i-1}{2}}\mathbb{E} \left\| \bcw_{0} - \bcw_{c,0} \right\| \\
        & +\mu \sqrt{C}\frac{1}{1-\sqrt{\xi}} \sqrt{K(G^2+\max_{k}\sigma_k^2)} \\
        \stackrel{(a)}{\le} & 2\mu \sqrt{C}\frac{1}{1-\sqrt{\xi}} \sqrt{K(G^2+\max_{k}\sigma_k^2)} \\
    \end{aligned}
\end{equation}

where (a) holds whenever:
\begin{equation}
    \begin{aligned}
    & \sqrt{C}\xi^{\frac{i-1}{2}}\mathbb{E} \left\| \bcw_{0} - \bcw_{c,0} \right\| \leq \mu \sqrt{C}\frac{1}{1-\sqrt{\xi}} \sqrt{K(G^2+\max_{k}\sigma_k^2)} \\
    & \Longleftrightarrow \xi^{\frac{i-1}{2}} \leq \mu \frac{1}{\left(1-\sqrt{\xi}\right)\mathbb{E} \left\| \bcw_{0} - \bcw_{c,0} \right\|} \sqrt{K(G^2+\max_{k}\sigma_k^2)} \\
    & \Longleftrightarrow \frac{i-1}{2} \log \left(\xi\right) \leq \log (\mu)+O(1) \\
    & \Longleftrightarrow i \geq \frac{2 \log (\mu)}{\log \left(\xi\right)}+O(1)
    \end{aligned}
\end{equation}

We conclude that all agents in the network will contract around the network centroid $\left(\bm{\phi}_i^{\top} \otimes I\right) \bcw_i$ after sufficient iterations $i_o$, where:
\begin{equation}
    i_o = \frac{2 \log (\mu)}{\log \left(\xi\right)}+O(1) \stackrel{(a)}{=}o(\mu^{-1})
\end{equation}
where (a) follows because $\lim_{\mu \to 0}\mu \log(\mu)=0$


\subsection{Proof of Theorem \ref{thm:descent}}
    From the assumption of Lipschitz gradients, we have
    \begin{equation}
        \begin{aligned}
            J(\bm{w}_{c,i}) \leq & J(\bm{w}_{c,i-1})+\nabla J(\bm{w}_{c,i-1})^{\top}\left(\bm{w}_{c,i}-\bm{w}_{c,i-1}\right)\\ 
            & +\frac{\delta}{2}\|\bm{w}_{c,i}-\bm{w}_{c,i-1}\|^2
        \end{aligned}
    \end{equation}
    We have the centroid recursion:
    \begin{equation}
        \label{eqn:centroidPert}
        \bm{w}_{c, i} = \bm{w}_{c,i-1}-\mu\sum_{k=1}^{K}[\bm{\phi}_{i-1}]_k \nabla J_k\left(\bm{w}_{c, i-1}\right) - \mu \bm{d}_{i-1}-\mu \bm{s}_i
    \end{equation}
    where we define the perturbation terms:
    \begin{equation}
        \label{eqn:pertD}
        \bm{d}_{i-1} \triangleq \sum_{k=1}^K [\bm{\phi}_{i-1}]_k\left(\nabla J_k\left(\bm{w}_{k, i-1}\right)-\nabla J_k\left(\bm{w}_{c, i-1}\right)\right) 
    \end{equation}
    \begin{equation}
        \label{eqn:pertS}
        \bm{s}_i \triangleq \sum_{k=1}^K [\bm{\phi}_{i-1}]_k\left(\widehat{\nabla J}_k\left(\bm{w}_{k, i-1}\right)-\nabla J_k\left(\bm{w}_{k, i-1}\right)\right)
    \end{equation}
    From (\ref{eqn:centroidPert}), we then obtain:
    \begin{equation}
        \begin{aligned}
            J\left(\boldsymbol{w}_{c, i}\right) \leq  & J\left(\boldsymbol{w}_{c, i-1}\right) -\mu\left\|\nabla J\left(\boldsymbol{w}_{c, i-1}\right)\right\|^2 \\
            & -\mu \nabla J\left(\boldsymbol{w}_{c, i-1}\right)^{\top}\left(\boldsymbol{d}_{i-1}+\boldsymbol{s}_i\right) \\
            & +\mu^2 \frac{\delta}{2}\left\|\nabla J\left(\boldsymbol{w}_{c, i-1}\right)+\boldsymbol{d}_{i-1}+\boldsymbol{s}_i\right\|^2
        \end{aligned}
    \end{equation}
    \begin{lemma}[Perturbation bounds] \label{lem:PertBounds}
        Under Assumptions \ref{assup:lip}–\ref{asup:gn[]} and for sufficiently small step-sizes $\mu$, the perturbation terms are bounded as:
        \begin{align}
            \label{eqn:thmdBound}
            \mathbb{E}\left\|\bm{d}_{i-1}\right\|^2 \leq &O\left(\mu^2\right) \\    
            \label{eqn:s2ndBound}
            \mathbb{E}\left\{\left\|\bm{s}_i\right\|^2 \mid \bcw_{i-1}\right\} \leq & K\max _k \sigma_k^2  \triangleq \sigma^2\\
        \end{align}
        after sufficient iterations $i \geq i_o$.
        \qed
    \end{lemma}
    
    \begin{proof}
        We begin by studying the perturbation term $\bm{s}_i$. We have:
        \begin{equation}
            \begin{aligned}
                & \mathbb{E}\left\{\left\|\boldsymbol{s}_i\right\|^2 \mid \bcw_{i-1}\right\} \\
                = & \mathbb{E}\left\{\left\|\sum_{k=1}^K [\bm{\phi}_{i-1}]_k\left(\widehat{\nabla J}_k\left(\boldsymbol{w}_{k, i-1}\right)-\nabla J_k\left(\boldsymbol{w}_{k, i-1}\right)\right)\right\|^2 \mid \bcw_{i-1}\right\} \\
                \stackrel{(a)}{=} & \mathbb{E}\left\{\sum_{k=1}^K [\bm{\phi}_{i-1}]_k^2 \left\|\widehat{\nabla J}_k\left(\boldsymbol{w}_{k, i-1}\right)-\nabla J_k\left(\boldsymbol{w}_{k, i-1}\right)\right\|^2 \mid \bcw_{i-1}\right\} \\
                = & \sum_{k=1}^K \mathbb{E}\left\{ [\bm{\phi}_{i-1}]_k^2 \left\|\widehat{\nabla J}_k\left(\boldsymbol{w}_{k, i-1}\right)-\nabla J_k\left(\boldsymbol{w}_{k, i-1}\right)\right\|^2 \mid \bcw_{i-1}\right\}\\
                \stackrel{(b)}{\leq} & \sum_{k=1}^K \mathbb{E}\left\{ \left\|\widehat{\nabla J}_k\left(\boldsymbol{w}_{k, i-1}\right)-\nabla J_k\left(\boldsymbol{w}_{k, i-1}\right)\right\|^2 \mid \bcw_{i-1}\right\}\\
                \stackrel{(c)}{\leq} & \sum_{k=1}^K \sigma_k^2 \leq K\max _k \sigma_k^2
            \end{aligned}
        \end{equation}
        where (a) follows from the uncorrelatedness assumption (\ref{eqn:uncorrelated}) and (b) follows because $[\bm{\phi}_{i-1}]_k^2 \leq 1$ with probability 1, and (c) follows from the bound on the second-order moment (\ref{eqn:gd}). For the second perturbation term, we have:
        \begin{equation}
            \label{eqn:dbound}
            \begin{aligned}
                & \left\|\boldsymbol{d}_{i-1}\right\|^2 \\
                = & \left\|\sum_{k=1}^K [\bm{\phi}_{i-1}]_k\left(\nabla J_k\left(\boldsymbol{w}_{k, i-1}\right)-\nabla J_k\left(\boldsymbol{w}_{c, i-1}\right)\right)\right\|^2 \\
                \stackrel{(a)}{\leq} & \sum_{k=1}^K [\bm{\phi}_{i-1}]_k\left\|\nabla J_k\left(\boldsymbol{w}_{k, i-1}\right)-\nabla J_k\left(\boldsymbol{w}_{c, i-1}\right)\right\|^2 \\
                \stackrel{(b)}{\leq} & \delta^2 \sum_{k=1}^K [\bm{\phi}_{i-1}]_k\left\|\boldsymbol{w}_{k, i-1}-\boldsymbol{w}_{c, i-1}\right\|^2 
            \end{aligned}
        \end{equation}
        where (a) again follows from Jensen's inequality, (b) follows from the Lipschitz gradient condition.
        
        We take the expectations to get when $i \geq i_o$:
        \begin{equation}
           \begin{aligned}
            & \mathbb{E} \left\|\boldsymbol{d}_{i-1}\right\|^2 \\
            \leq & \delta^2 \sum_{k=1}^K \mathbb{E} \left\{[\bm{\phi}_{i-1}]_k\left\|\boldsymbol{w}_{k, i-1}-\boldsymbol{w}_{c, i-1}\right\|^2\right\} \\
            \stackrel{(a)}{\leq} & \delta^2 \sum_{k=1}^K \mathbb{E} \left\{\left\|\boldsymbol{w}_{k, i-1}-\boldsymbol{w}_{c, i-1}\right\|^2\right\} \\
            \stackrel{(\ref{eqn:disg2nd})}{\leq} & O(\mu^2)
            \end{aligned} 
        \end{equation}
        where (a) follows because $[\bm{\phi}_{i-1}]_k^2 \leq 1$ with probability 1.
    \end{proof}
    
    We define $\mathcal{G} \triangleq\left\{w:\|\nabla J(w)\|^2 \geq \mu \frac{c_2}{c_1}\left(1+\frac{1}{\pi}\right)\right\}$ where $0 < \pi < 1$ is a parameter to be chosen and 
    \begin{align}
        c_1 & \triangleq \frac{1}{2}(1-2\mu\delta)=O(1) \\
        c_2 & \triangleq \frac{\delta\sigma^2}{2}=O(1)
    \end{align}

    Following (87) in \cite{Vlaski21P1}, we take expectations conditioned on $\bm{w}_{c,i-1} \in \mathcal{G}$:
    \begin{equation}
        \begin{aligned}
            & \mathbb{E}\{J(\bm{w}_{c,i})|\bm{w}_{c,i-1} \in \mathcal{G}\} \\
            \leq & \mathbb{E}\{J(\bm{w}_{c,i-1})|\bm{w}_{c,i-1} \in \mathcal{G}\} - \mu^2 c_1\frac{c_2}{c_1}\left(1+\frac{1}{\pi}\right) \\
            & +\frac{\mu}{2}\left(1+2\mu\delta\right)\mathbb{E}\{\|\bm{d}_{i-1}\|^2|\bm{w}_{c,i-1}\in\mathcal{G}\}+\mu^2 c_2 \\
            \leq & \mathbb{E}\{J(\bm{w}_{c,i-1})|\bm{w}_{c,i-1} \in \mathcal{G}\} \\
            &+\frac{\mu}{2}\left(1+2\mu\delta\right)\mathbb{E}\{\|\bm{d}_{i-1}\|^2|\bm{w}_{c,i-1}\in\mathcal{G}\} -\mu^2 \frac{c_2}{\pi} \\
            \stackrel{(a)}{\leq} & \mathbb{E}\{J(\bm{w}_{c,i-1})|\bm{w}_{c,i-1} \in \mathcal{G}\} + \frac{\mu}{2}\left(1+2\mu\delta\right) \frac{O(\mu^2)}{\pi_{i-1}^{\mathcal{G}}} -\mu^2 \frac{c_2}{\pi} \\
            \leq & \mathbb{E}\{J(\bm{w}_{c,i-1})|\bm{w}_{c,i-1} \in \mathcal{G}\} + \frac{O(\mu^3)}{\pi_{i-1}^{\mathcal{G}}}-\mu^2 \frac{c_2}{\pi} \\
            \leq & \mathbb{E}\{J(\bm{w}_{c,i-1})|\bm{w}_{c,i-1} \in \mathcal{G}\}
        \end{aligned}
    \end{equation}
    where (a) follows (91) in \cite{Vlaski21P1}

\end{document}